\documentclass[10pt,letterpaper]{article}
\usepackage[margin=0.9in]{geometry}
\usepackage{amsmath,amsthm}
\usepackage{newtxtext,newtxmath}
\usepackage{microtype,booktabs,graphicx,enumitem,array}
\usepackage[dvipsnames]{xcolor}
\usepackage{hyperref,fancyhdr}
\hypersetup{colorlinks=true,linkcolor=MidnightBlue,citecolor=MidnightBlue,urlcolor=MidnightBlue,
 pdftitle={Marginally Correct Tool Caches Can Reverse Group-Normalized Policy Updates},
 pdfauthor={Shivam Gupta},pdfsubject={Stochastic tool caching and group-relative learning}}
\setlist{nosep,leftmargin=1.5em}
\newtheorem{theorem}{Theorem}
\newtheorem{proposition}[theorem]{Proposition}
\newtheorem{corollary}[theorem]{Corollary}

\newcommand{\E}{\mathbb{E}}
\newcommand{\Var}{\operatorname{Var}}
\newcommand{\Bern}{\operatorname{Bern}}
\newcommand{\Bin}{\operatorname{Bin}}
\newcommand{\sgn}{\operatorname{sgn}}
\newcommand{\ind}{\mathrm{ind}}
\newcommand{\shr}{\mathrm{share}}

\title{\vspace{-1.4em}\LARGE\bfseries Marginally Correct Tool Caches Can Reverse\\ Group-Normalized Policy Updates}
\author{\large Shivam Gupta\thanks{LLM-based tools assisted with research exploration, code, mathematical analysis, manuscript preparation, and internal review. Internal automated reviews are not external peer review.}\\[3pt]\normalsize Independent Researcher\\
\normalsize\href{mailto:shivam1720406@gmail.com}{shivam1720406@gmail.com}}
\date{20 September 2026}

\begin{document}
\maketitle
\thispagestyle{empty}
\begin{abstract}
Tool-result caching reduces repeated execution in agent training, but also
couples rollout randomness. We study a two-action model in which independent
and shared execution preserve every rollout's conditional reward distribution.
Despite this marginal agreement, sharing one stochastic result per group can
reverse the expected group-normalized policy update. We derive an exact
finite-group expression: against a constant alternative, the shared update
follows the probability of winning minus the probability of losing, rather
than the difference in expected reward. A Bernoulli specialization yields a
wrong-direction region and a non-vanishing update-variance floor as group size
grows. Centering without group standard-deviation scaling preserves the
expected-return direction in this model, using an existing estimator control.
Exhaustive finite sums verify 540 configurations and 3,240 estimator evaluations,
with a separate ordered-sequence checker. An implementation audit reproduces
the sharing path in a pinned, unmodified TVCache stack using 256 scripted
rollouts. These results do not measure language-model training performance or
refute TVCache's deterministic-output contract. They establish that marginal
output validity alone cannot certify a stochastic cache as training-equivalent.
\end{abstract}

\section{Introduction}
An execution optimization can change a learner without changing its model,
prompt, reward function, or the marginal quality of any one tool response.
Consider a training group containing several identical calls to a stochastic
tool. Executing each call draws several outcomes. Executing one call and
broadcasting its result draws one outcome and repeats it. Both procedures can
return the correct distribution to each individual caller, while producing
different joint distributions across the group. A learner that compares
rollouts within that group can distinguish the two procedures.

This distinction matters for group-relative training. GRPO forms advantages
from rewards centered and scaled within a group of completions
\cite{deepseekmath}. Stateful tool caches address a different correctness
question: whether reuse respects the environment state that produced an output
\cite{tvcache}. State consistency is essential, but does not by itself preserve
the sampling law of a stochastic tool. This paper asks a narrower question:
\emph{if a cache preserves each rollout's conditional reward distribution, must
it preserve the expected group-normalized update?}

The answer is no. Our counterexample uses an action with a constant payoff and
an action with a random payoff. Under shared execution, all occurrences of the
random action receive the same newly sampled result within a group. The cache
is refreshed for every group. Consequently the effect requires neither stale
state nor a permanently unlucky cached answer. The normalizer removes the
magnitude of the realized payoff gap, leaving only its sign. For Bernoulli
payoffs, the resulting preference threshold is $1/2$, even when the alternative
has a substantially higher expected payoff.

Group-normalization bias and centering-based alternatives are established
\cite{drgrpo,che}. We study how a controlled change in tool execution induces
that bias, with exact statistical consequences and an auditable implementation
path. Our contributions are:
\begin{enumerate}
\item A finite-group characterization of shared-outcome updates for an arbitrary
scalar reward distribution, with a Bernoulli sign-reversal example.
\item A separation between marginal reward agreement and update agreement,
including asymptotic means, a persistent variance floor, and an existing
centered-estimator control.
\item Exhaustive numerical checks and a pinned runtime audit that keep
mathematical claims, scripted execution evidence, and untested deployment
implications separate.
\end{enumerate}

The engineering motivation is practical: a lower physical tool-call count is
not, by itself, evidence of equivalent training. The results identify a
condition to test when optimizing expensive stochastic tools. They do not
measure production prevalence, model accuracy loss, or monetary savings.

\section{Related work and scope}
\paragraph{Group-relative optimization.}
DeepSeekMath introduces GRPO as a critic-free group-relative policy optimization
method \cite{deepseekmath}. Liu et al. identify biases associated with
group-level standard-deviation scaling and response-length normalization, and
propose Dr.\ GRPO \cite{drgrpo}. Our centered-only control uses this existing
principle. We analyze a one-step, on-policy score estimator, not the complete
clipped, multi-epoch, token-weighted training algorithm.

\paragraph{Reward magnitude and stochastic feedback.}
Che et al. show that in a sparse-answer regime, group normalization can erase
the magnitude of an error penalty and change a decision threshold \cite{che}.
Our Bernoulli independent-execution model is affinely equivalent to their three-reward
setting. Sharing makes the relevant two-level reward structure arise at every
interior action probability, rather than only in a sparse-action regime. This
close relationship limits any claim of conceptual novelty. El Mansouri et al.
study correction for noisy rewards \cite{noisecorrected}; Xin measures
within-group verifier-error dependence \cite{xin}. Here the outcomes need not
be erroneous: the intervention changes their dependence directly.

\paragraph{Tool caching.}
TVCache reuses tool results across stateful trajectories \cite{tvcache}. Its
Appendix B assumes that tool outputs depend only on sandbox state and arguments.
Our stochastic setting is outside that deterministic-output contract, so the
counterexample does not refute its correctness argument. CacheRL studies cached
agent-training environments, including cache-aware reward design and masking
of injected tool tokens \cite{cacherl}. Correctly excluding environment tokens
from the policy loss does not, by itself, establish equivalence of the joint
reward law. We neither reproduce those papers' training results nor claim that
their reported benchmarks exhibit our effect.

\section{Model and estimators}
Let $G\geq2$ be the rollout-group size. Each rollout independently chooses
$a_i\sim\Bern(p)$, where $0<p<1$ and $p=\sigma(\theta)$ is parameterized by a
logit $\theta$. Action $a_i=0$, denoted A, returns the constant $c$. Action
$a_i=1$, denoted B, returns a scalar reward $Y$ with finite first moment
$\mu=\E[Y]$. Actions are independent of tool randomness. The expected return
and its gradient are
\begin{equation}
J(\theta)=(1-p)c+p\mu,\qquad
J'(\theta)=p(1-p)(\mu-c).
\label{eq:target}
\end{equation}
The policy score for a sampled action is $\partial_\theta\log\pi_\theta(a_i)=a_i-p$.

\paragraph{Two execution laws.}
Under \emph{independent execution}, each B invocation receives an independent
copy $Y_i$. Under \emph{shared execution}, one $Y$ is drawn independently for
the group and reused for every B invocation. Drawing on the first B cache miss
is equivalent to drawing in advance because actions are independent of $Y$.
There is no persistence between groups. In both modes,
\begin{equation}
\mathcal L(r_i\mid a_i=0)=\delta_c,\qquad
\mathcal L(r_i\mid a_i=1)=\mathcal L(Y).
\label{eq:marginals}
\end{equation}
Thus the intervention preserves conditional reward marginals, not only their
means. When two rollouts choose B, their outcomes are independent in the first
mode and identical in the second. This is the only distributional change.

\paragraph{Updates.}
Write $\bar r=G^{-1}\sum_i r_i$ and
$s_r^2=G^{-1}\sum_i(r_i-\bar r)^2$. We study the ascent estimators
\begin{align}
U_\epsilon&=\frac1G\sum_{i=1}^G
 (a_i-p)\frac{r_i-\bar r}{s_r+\epsilon},\qquad \epsilon\geq0,
\label{eq:norm}\\
V&=\frac1G\sum_{i=1}^G(a_i-p)(r_i-\bar r).
\label{eq:centered}
\end{align}
At $s_r=0$, $U_0$ is defined to be zero. Positive updates increase the
probability of B. The denominator is the population standard deviation within
the group. Using the sample standard deviation changes the scale of $U_0$ by
a positive, group-size-dependent factor, and therefore does not change its
sign. We make no such scale-equivalence claim for nonzero $\epsilon$ without
adjusting that constant.

These estimators describe the reward component of an on-policy update at unit
importance ratio. We do not model KL penalties, clipping after parameter
movement, adaptive optimizers, variable token lengths, or multi-turn feedback.

\section{Exact consequences of sharing}
Let $N=\sum_i a_i\sim\Bin(G,p)$ and define
\begin{equation}
w_G(n)=\frac{\sqrt{n(G-n)}}{G},\quad
h_G(n)=w_G(n)^2,\quad S_G(p)=\E[w_G(N)].
\label{eq:weights}
\end{equation}
For interior $p$ and $G\geq2$, $S_G(p)>0$.

\begin{theorem}[Shared outcomes replace payoff magnitude by ordering]
Under shared execution, for $\epsilon=0$,
\begin{equation}
U_0=w_G(N)\sgn(Y-c),\qquad
\E[U_0]=S_G(p)\bigl(\Pr(Y>c)-\Pr(Y<c)\bigr),
\label{eq:general}
\end{equation}
where $\sgn(0)=0$. Consequently, conditional marginal agreement in
\eqref{eq:marginals} does not ensure agreement with the expected-return direction.
\end{theorem}
\begin{proof}
Centered rewards sum to zero, so the $-p$ score terms cancel. The centered
numerator equals $h_G(N)(Y-c)$. The shared group contains only the two reward
values $c$ and $Y$, giving $s_r=w_G(N)|Y-c|$. Division yields
\eqref{eq:general} when $0<N<G$ and $Y\ne c$. Both sides are zero otherwise.
Finally $N$ and $Y$ are independent, so their expectations factor.
\end{proof}

The theorem concerns the \emph{direction of the update}, not a claim that an
arbitrary full learning algorithm maximizes a median objective. It separates
the frequency of beating an alternative from the value of doing so. A small
gain on many calls and a large loss on fewer calls can have positive ordering
balance but negative mean advantage.

\begin{corollary}[Bernoulli threshold and wrong-direction region]
For $Y\sim\Bern(q)$ and $0<c<1$,
\begin{equation}
\E[U_{0,\shr}]=(2q-1)S_G(p).
\label{eq:bernoulli}
\end{equation}
For every $G\geq2$ and interior $p$, the shared normalized update has the
opposite sign to $J'$ whenever $1/2<q<c<1$ or $0<c<q<1/2$.
\end{corollary}
The threshold is $q=1/2$ for the shared update and $q=c$ for expected return.
This statement does not say that independent normalization is unbiased.
At $G=2$, both modes have the same expected update: a group containing both
actions includes only one stochastic B result.

\begin{proposition}[Centering preserves the mean-gradient direction]
For either execution law and any integrable $Y$,
\begin{equation}
\E[V]=\left(1-\frac1G\right)p(1-p)(\mu-c).
\label{eq:centeredmean}
\end{equation}
Multiplication by $G/(G-1)$ removes the self-including-baseline factor.
\end{proposition}
\begin{proof}
Conditional on $N=n$, the centered numerator is
$(G-n)\bigl(\sum_{i:a_i=1}r_i-nc\bigr)/G^2$. Both laws give conditional sum
expectation $n\mu$. Apply $\E[N(G-N)]=G(G-1)p(1-p)$.
\end{proof}
This is an existing baseline correction. It does not restore the entire
distribution of the update, nor establish equivalence after nonlinear clipping
or adaptive optimizer transformations.

\begin{theorem}[A persistent mean shift and variance floor]
\label{thm:variance}
Suppose $Y\sim\Bern(q)$ with $0<q<1$, $0<c<1$, and fixed interior $p$.
As $G\to\infty$,
\begin{align}
\E[U_{0,\ind}]&\longrightarrow
\frac{p(1-p)(q-c)}{\sqrt{p q(1-q)+p(1-p)(q-c)^2}},
\label{eq:freshlimit}\\
\E[U_{0,\shr}]&\longrightarrow (2q-1)\sqrt{p(1-p)},
\label{eq:sharelimit}\\
\Var(U_{0,\ind})&\longrightarrow0,\qquad
\Var(U_{0,\shr})\longrightarrow4q(1-q)p(1-p).
\label{eq:varlimit}
\end{align}
At finite $G$, the shared variance is exactly
\begin{equation}
\Var(U_{0,\shr})=(1-1/G)p(1-p)-(2q-1)^2S_G(p)^2.
\label{eq:finitevar}
\end{equation}
\end{theorem}
\begin{proof}
Under independent execution, laws of large numbers give the covariance
$p(1-p)(q-c)$ and reward variance $pq(1-q)+p(1-p)(q-c)^2$.
Under sharing, use \eqref{eq:general} and $N/G\to p$. In either mode,
Cauchy--Schwarz gives $|U_0|\leq\sqrt{(N/G)(1-N/G)}\leq1/2$.
Bounded convergence therefore applies to first and second moments. For the
shared second moment, $(Y-c)$ is never zero, so $U_0^2=h_G(N)$ and
$\E[U_0^2]=(1-1/G)p(1-p)$. Subtracting the squared mean proves
\eqref{eq:finitevar} and its limit.
\end{proof}
In the wrong-direction regions, sufficiently large independent groups and
shared groups point in opposite directions. Increasing the number of policy
samples does not increase the number of independent tool outcomes in the
shared group. Independent groups can still average down this noise; the result
is a within-group variance floor, not irreducible uncertainty across all data.

\subsection{Nonzero numerical stabilization}
For $\epsilon>0$ in the Bernoulli setting, define
\begin{equation}
A_\epsilon=\E\!\left[\frac{h_G(N)(1-c)}{w_G(N)(1-c)+\epsilon}\right],\qquad
B_\epsilon=\E\!\left[\frac{h_G(N)c}{w_G(N)c+\epsilon}\right].
\end{equation}
Then $\E[U_{\epsilon,\shr}]=qA_\epsilon-(1-q)B_\epsilon$ and the zero-update
threshold is $q_\epsilon=B_\epsilon/(A_\epsilon+B_\epsilon)$.
For $c>1/2$, $1/2<q_\epsilon<c$; for $c<1/2$, $c<q_\epsilon<1/2$.
These inequalities follow termwise by comparing $A_\epsilon$ and
$B_\epsilon$, and $(1-c)B_\epsilon$ with $cA_\epsilon$.
As $\epsilon\downarrow0$ the threshold tends to $1/2$; as
$\epsilon\to\infty$ it tends to $c$. A nonzero stabilizer therefore does not
automatically restore the intended preference threshold. Appendix~\ref{app:epsilon}
states the endpoint conventions.

\section{Exhaustive numerical verification}
\subsection{Protocol and controls}
Before executing the sweep, we fixed $G\in\{2,4,8,16,32,64\}$,
$p\in\{0.1,0.5,0.9\}$, $c\in\{0.1,0.3,0.5,0.7,0.9\}$, and
$q\in\{0,0.2,0.4,0.6,0.8,1\}$. This gives 540 configurations.
For each configuration we evaluate both execution modes with centering alone
and with normalization at $\epsilon\in\{0,10^{-4}\}$, for 3,240 estimator
evaluations. These are finite probability-weighted sums computed in floating
arithmetic, not Monte Carlo observations or independent trials.

With $N=n$, let $K$ count successful B outcomes. Independent execution uses
$K\mid N=n\sim\Bin(n,q)$; shared execution uses $K\in\{0,n\}$ with
probabilities $1-q,q$. The centered numerator is
\begin{equation}
V(n,k)=\frac{(G-n)(k-nc)}{G^2}.
\end{equation}
The implementation computes the reward variance from centered squared
deviations, avoiding subtraction of nearly equal second moments. It accumulates
weighted sums with compensated summation.

A separately implemented checker enumerates ordered rollout outcomes and
evaluates the literal score expression, rather than the count reduction. It
checks 288 small-group combinations. Other controls verify probability mass,
the reward mean, the centered identity, the shared formula, deterministic
$q\in\{0,1\}$ equivalence, and $G=2$ equivalence. All six numerical test methods pass.
Maximum probability-mass error is below $3.34\times10^{-15}$; maximum marginal-mean
error is below $1.78\times10^{-15}$. Numerical precision checks are not statistical
confidence intervals.

\subsection{A preselected witness}
The protocol selected $p=0.5,c=0.9,q=0.8$ analytically before execution.
Both laws have mean reward $0.85$ and true gradient $-0.025$. At $G=64$,
independent normalization gives $-0.078867$, while sharing gives $+0.297628$.
The centered estimator gives $-0.024609$ under either law
(Table~\ref{tab:witness} and Figure~\ref{fig:updates}). Thus the intervention reverses the expected update
while leaving the conditional reward marginals unchanged.

\begin{table}[ht]
\centering
\caption{Exact finite-sum witness, rounded to six decimals. Positive updates
increase B, the lower-mean action. Normalized columns use $\epsilon=0$.}
\label{tab:witness}
\begin{tabular}{rrrr}
\toprule
$G$ & Independent normalized & Shared normalized & Centered, either law\\
\midrule
2 & 0.150000 & 0.150000 & -0.012500\\
4 & 0.187129 & 0.242404 & -0.018750\\
8 & 0.111526 & 0.278348 & -0.021875\\
16 & 0.005998 & 0.290118 & -0.023438\\
32 & -0.059543 & 0.295196 & -0.024219\\
64 & -0.078867 & 0.297628 & -0.024609\\
\bottomrule
\end{tabular}
\end{table}

\begin{figure}[ht]
\centering
\includegraphics[width=\linewidth]{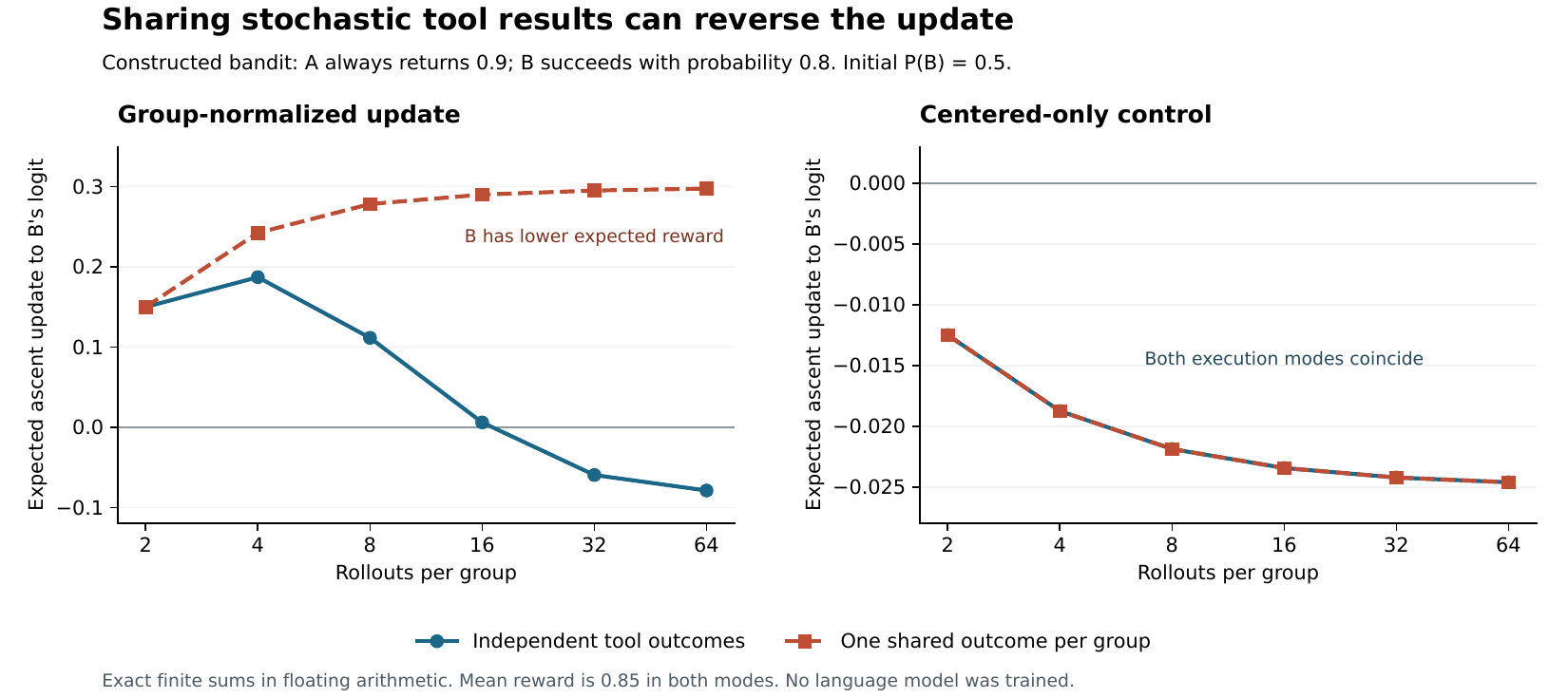}
\caption{Increasing rollout count separates the two normalized updates in the
preselected witness. Centering alone preserves their expected direction.
Panels have different vertical scales. Curves are exhaustive calculations,
not learning curves or measured model performance.}
\label{fig:updates}
\end{figure}

Across the entire specified grid, the modes have opposite normalized-update
signs in 54 of 540 configurations at each tested $\epsilon$. Independent
normalization opposes the expected-return gradient in 54 configurations, and
shared normalization in 108. Centering alone has no such sign discrepancy.
The deterministic controls agree exactly. The grid was constructed to examine
the mechanism; its fractions must not be interpreted as deployment prevalence.
The $\epsilon=10^{-4}$ result verifies robustness to one practical numerical
stabilizer, not all implementation conventions.

\section{A pinned implementation audit}
\subsection{Execution boundary}
To connect the intervention to executable infrastructure, we inspected and
ran TVCache revision \texttt{3a4f95a6582e}
\cite{tvcachecode}. The adaptive follow-up protocol was written after source
inspection and after the numerical study, before executing the runtime probe.
The probe uses unmodified implementations of the semantic stateful executor,
async client, Flask routes, and immutable prefix-tree backend. Only the HTTP
transport is replaced with an in-process bridge to Flask's test client. A
scripted environment implements the tool interface. Socket connections and DNS
resolution are disabled during execution.

Each case has 32 sequential one-call rollouts, alternating A and B, with a new
executor for each rollout and a shared backend within the case. A returns the
string \texttt{0.9}. B either always returns \texttt{0.8}, or alternates
\texttt{1}, \texttt{0} on physical execution, starting with \texttt{1}.
The latter is a deliberately varying stream, not a Bernoulli sample.
We compare direct execution, shared identity, a separate task namespace per
rollout, and a separate draw identifier in tool arguments.

\begin{table}[ht]
\centering
\caption{Runtime audit. Counts are per case of 32 scripted rollouts. Each mode
is run once with constant outputs and once with a varying stream.}
\label{tab:runtime}
\begin{tabular}{lrl}
\toprule
Mode & Physical calls & Returned B stream in varying fixture\\
\midrule
Direct execution & 32 & Alternating 1, 0\\
Shared identity & 2 & Sixteen copies of 1\\
Separate task namespace & 32 & Alternating 1, 0\\
Separate draw identity & 32 & Alternating 1, 0\\
\bottomrule
\end{tabular}
\end{table}

\subsection{Observed behavior and interpretation}
All eight cases satisfy the protocol's expectations (Table~\ref{tab:runtime}).
The 256 scripted rollouts include 64 direct-environment controls and 192
rollouts through the cache stack. In total there are 196 physical tool calls
and 648 in-process HTTP requests. The socket and DNS guards record zero
attempts. Three full runs in separate processes produce byte-identical result
JSON. Before import, all 29 exported source files pass hash verification.

The shared case executes only one A call and one B call. Reuse preserves the
constant-output control and repeats the first B result in the varying fixture.
Identity-separated controls retain all physical executions. These controls
trade away reuse; they are not a new or cost-free fix. The experiment establishes
an implementation path for sharing, not a failure of the vendor's stated
deterministic contract or a measured training regression.

The inspected video integration labels loading and preprocessing as
state-changing and includes a model-backed caption call. This motivates a
separate variability measurement, which we have not performed. Importantly,
TVCache's paper lists its EgoSchema training objective as importance sampling,
while other workloads use GRPO \cite{tvcache}. We therefore do not infer that
its caption workload uses the group-normalized estimator studied here.

\section{Implications, limitations, and reproducibility}
\paragraph{Cache equivalence is estimator-dependent.}
For the mean-centered estimator in this model, sharing preserves the expected
gradient up to the usual finite-group factor. For the normalized estimator, it
need not. An appropriate training audit must specify both the execution law and
the consuming estimator. Matching tool-response means or individual output
distributions alone is weaker than matching the expected update.

\paragraph{Cost must be reported separately.}
In the one-step Bernoulli model, independent execution requires $Gp$ expected
physical B calls per group, whereas sharing requires $1-(1-p)^G$. That saving
comes with different gradient variance. Equal-rollout comparisons and
equal-physical-call comparisons answer different questions. Our work reports
no time-to-accuracy or dollar-efficiency result. Appendix~\ref{app:centeredvar}
gives exact centered-estimator variances to support such comparisons without
assuming that identical expected gradients imply identical learning behavior.

\paragraph{Scope of the evidence.}
There are two actions, one decision per rollout, a fixed policy, and a constant
alternative. Theorems concern a specific on-policy reward estimator. We do not
train or evaluate a language model, establish an end-to-end convergence result,
or measure the frequency of stochastic tools in deployed agents. Actual
multi-turn policies can respond to cached observations, changing their action
distribution as well as their rewards. Cache lifetime, concurrent misses,
partial sharing, eviction, and stale state can introduce additional effects
outside the model. A production result would require representative tool
distributions, measured reuse, repeated training runs, and equal-budget
comparisons with existing optimizers. Automated internal review does not
replace independent scientific review or the author's responsibility for the
manuscript.

\paragraph{Reproducibility.}
The artifact contains pre-execution protocols, the independent checker, all
finite-sum results, runtime traces, source and result hashes, exact dependency
versions, and the figure-generation code. The numerical study requires only
Python's standard library; plotting uses Matplotlib. The runtime probe uses
Python 3.12 with Flask 3.1.2, HTTPX 0.28.1, and Requests 2.32.5, with all
transitive versions pinned. Third-party code is retrieved from the pinned Git
revision, kept unmodified, and excluded from redistribution. No credentials,
paid API calls, participant data, or model outputs are required to reproduce
the reported experiments. The public artifact, including source and numerical
results, is available at \url{https://github.com/shi1720/tool-cache-coupling}.

\paragraph{Broader impact.}
The intended use is auditing training optimizations before relying on their
efficiency claims. A marginally valid cache can be useful and still change the
learning problem. Conversely, this constructed counterexample should not be
used to assert that tool caching is generally unsafe or that a named system's
published results are invalid.

\section{Conclusion}
Sharing a stochastic tool result changes the joint sampling law seen by a
group-relative learner. In a controlled model with identical conditional reward
marginals, the shared normalized update follows an ordering comparison rather
than the mean-payoff advantage. Sharing can reverse the expected normalized
update and leave a
within-group variance floor as rollout count grows. Exhaustive calculations and
a pinned runtime audit make the mechanism reproducible. The practical lesson
is specific: evaluate stochastic reuse against the estimator it serves, and
measure physical tool outcomes separately from rollout count.

\clearpage
\appendix
\section{Exact summation and equality controls}
For $b_G(n;p)=\binom Gn p^n(1-p)^{G-n}$, the independent expected update is
\begin{equation}
\E[U_{\epsilon,\ind}]=\sum_{n=0}^G b_G(n;p)
\sum_{k=0}^n\binom nk q^k(1-q)^{n-k}
\frac{(G-n)(k-nc)}{G^2(s(n,k)+\epsilon)},
\end{equation}
where the fraction is zero for a zero denominator and
\begin{align}
m(n,k)&=((G-n)c+k)/G,\\
s(n,k)^2&=\bigl[(G-n)(c-m)^2+k(1-m)^2+(n-k)m^2\bigr]/G.
\end{align}
The shared sum replaces the inner binomial distribution with mass $1-q$ at
$k=0$ and mass $q$ at $k=n$. When $n=0$, only one outcome is counted, with
total mass one. The numerical implementation stores the update expectation and variance,
probability mass, and expected reward for every estimator.

The deterministic controls $q=0$ and $q=1$ give the same reward vector under
both laws, conditional on the action vector. Their update distributions agree,
not just their means. For $G=2$, groups with equal actions have identical
scores and a zero centered update even if their rewards differ. Mixed-action
groups have one B sample, so the update distributions agree in that case too.
These facts explain the zero differences observed by the controls.

\section{Variance of the centered estimator}
\label{app:centeredvar}
In the Bernoulli model, condition on $N$ and write $h=h_G(N)$. The law of total
variance gives
\begin{align}
\Var(V_{\shr})&=q(1-q)\E[h^2]+(q-c)^2\Var(h),\\
\Var(V_{\ind})&=q(1-q)\E\!\left[\frac{N(G-N)^2}{G^4}\right]
 +(q-c)^2\Var(h).
\end{align}
The conditional means are the same, $h(q-c)$. Sharing replaces the conditional
variance of a sum of $N$ independent Bernoulli outcomes, $Nq(1-q)$, by
$N^2q(1-q)$. Substitution into the squared coefficient $(G-N)^2/G^4$ proves
the formulas. Correcting the baseline factor multiplies either variance by
$G^2/(G-1)^2$.

These expressions make an important distinction explicit: a mean-direction
repair does not preserve gradient noise. Any efficiency claim must account for
the number of independently sampled groups and the physical call budget.
We do not turn these identities into a wall-clock or convergence claim.

\section{Stabilizer threshold and general reward ties}
\label{app:epsilon}
For $\epsilon>0$, $A_\epsilon$ and $B_\epsilon$ are strictly positive at
interior $p$ and $c$. If $c>1/2$, then for every $0<n<G$, the corresponding
summands satisfy $B_n>A_n$ and $(1-c)B_n<cA_n$. The first follows because
$x/(w x+\epsilon)$ is increasing in positive $x$; the second follows by
comparing the denominators after the common numerator $h c(1-c)$ is factored
out. Summing preserves both strict inequalities, yielding
$1/2<q_\epsilon<c$. The reversed inequalities apply when $c<1/2$; when
$c=1/2$, the threshold is exactly $1/2$. Boundary action counts contribute zero.
Finite sums permit both limits in $\epsilon$ to be taken termwise.

For general shared $Y$, ties at $Y=c$ contribute zero. The exact second moment
is $\E[U_0^2]=(1-1/G)p(1-p)\Pr(Y\ne c)$. Thus the variance extension of
Theorem~\ref{thm:variance} subtracts
$S_G(p)^2[\Pr(Y>c)-\Pr(Y<c)]^2$ from this expression. Finite first moment is
needed to compare with expected return; the bounded normalized update itself
is well-defined for any almost surely finite scalar $Y$.

\section{Runtime audit boundaries}
The runtime audit tests a narrow execution path: sequential one-call rollouts,
no stored sandbox forks, no eviction, no expiry, and no production network.
The backend's cleanup threads are stopped after each case. Distinct identities
are supplied by the experiment, not discovered by an agent. The varying-output
fixture is not the probabilistic environment used in the theorem, and its
returned means must not be used as evidence for marginal agreement. That
agreement is proved and numerically verified only in the analytical study.

\end{document}